\documentclass[submission,copyright,creativecommons]{eptcs}
\providecommand{\event}{ICLP 2023} % Name of the event you are submitting to

\usepackage{iftex}

  \usepackage{underscore}         % Only needed if you use pdflatex.
  \usepackage[T1]{fontenc}        % Recommended with pdflatex
\usepackage{times}
\usepackage{soul}
\usepackage{url}

\usepackage[utf8]{inputenc}
\usepackage[small]{caption}
\usepackage{graphicx,picinpar,blindtext,wrapfig}
\usepackage{amsmath}
\usepackage{amsthm}
\usepackage{booktabs}
\usepackage{algorithm}
\usepackage{algorithmic}
\usepackage{censor}
\usepackage{enumitem}
\usepackage{thm-restate}

\theoremstyle{definition}

\newtheorem{example}{Example}

\def\spock{{$\textit{spock}$}}
\def\dwasp{{$\textit{DWASP}$}}

\usepackage{amssymb}
\usepackage{amsmath}
\usepackage{listings}
\usepackage{tikz}
\usepackage{xspace}

\newcommand{\tuple}[1]{\ensuremath{\langle{#1}\rangle}\xspace}
\newcommand{\eval}[3]{\ensuremath{[\![{#1}]\!]_{#2}^{#3}}\xspace}

\makeatletter
\lst@Key{countblanklines}{true}[t]{\lstKV@SetIf{#1}\lst@ifcountblanklines}
\lst@AddToHook{OnEmptyLine}{%
    \lst@ifnumberblanklines\else%
    \lst@ifcountblanklines\else%
    \advance\c@lstnumber-\@ne\relax%
    \fi%
    \fi}
\makeatother
\lstdefinelanguage{asp}{
    breakatwhitespace=true,
    captionpos=b,
    numbers=left,
    numbersep=5pt,
    numberblanklines=false,
    countblanklines=false,
    commentstyle=\colour{gray},
    frame=bt, framexbottommargin=5pt, framextopmargin=5pt,
    aboveskip=5pt, belowskip=5pt,
    abovecaptionskip=10pt
}
\lstnewenvironment{asp}{
    \lstset{
        language=asp,
        showstringspaces=false,
        keepspaces=true,
        formfeed=\newpage,
        tabsize=4,
        numbers=none,
        breaklines=true,
        literate={~} {$\sim$}{1},
        frame=none,
    }
}{}
\def\naf{\ensuremath{\mathit{not}}\ \xspace}

\newcommand{\system}[1]{\ensuremath{\mathtt{#1}}\xspace}

\def\xasptwo{\system{xASP2}\xspace}

\def\xasp{\system{xASP}}
 
\def\xclingo{\system{xclingo}} 
\def\discasp{\system{DiscASP}} 
\def\expaspc{\system{exp(ASP^c)}}
\def\scasp{\system{s(CASP)}}  
\def\spock{\system{spock}}
\def\dwasp{\system{DWASP}}
\def\visualdlv{\system{Visual\text{-}DLV}}

\def\labas{\system{LABAS}}

\title{Explanations for Answer Set Programming}

\author{Mario Alviano
\institute{DEMACS, University of Calabria, \\ Via Bucci 30/B, 87036 Rende (CS), Italy}
\email{mario.alviano@unical.it}
\and
Ly Ly Trieu \& Tran Cao Son
\institute{New Mexico State Universty\\
NM, USA}
\email{\quad lytrieu|stran@nmsu.edu}
\and
Marcello Balduccini
\institute{Saint Joseph's University \\PA , USA}
\email{mbalducc@sju.edu}
}
\def\titlerunning{Explanations for Answer Set Programming}
\def\authorrunning{M. Alviano, L.L. Trieu, T.C. Son \& M. Balduccini}

\begin{document}
\maketitle

\begin{abstract}
The paper presents an enhancement of \xasp{}, a system that generates explanation graphs for Answer Set Programming (ASP). 
 Different from \xasp{}, the new system, \xasptwo,  supports different {\small \tt clingo} constructs like the choice rules, the constraints, and the aggregates such as $\#sum$, $\#min$.
This work formalizes and presents an explainable artificial intelligence system for a broad fragment of ASP, capable of shrinking as much as possible the set of assumptions and presenting explanations in terms of directed acyclic graphs.

\end{abstract}

\section{Introduction}\label{sec:introduction}

Recently, many modern artificial intelligence systems are increasingly capable of tackling complex problems. However, their lack of transparency can create a new issue: users may not comprehend why a solution was obtained, making it difficult to trust the results. 
Moreover, with the \emph{right to an explanation} law extensively discussed in the USA, EU, and UK, and partly enacted in some countries, explainable artificial intelligence (XAI) has experienced a substantial increase in interest.
Thus, the focus of this paper is on the development of an XAI system for Answer Set Programming (ASP) \cite{MarekT99,Niemela99}.
Answer Set Programming (ASP) \cite{MarekT99,Niemela99} is a well-known paradigm for problem-solving using logic programs under answer set semantics \cite{GelfondL90} in knowledge representation and reasoning (KR\&R) and an extension of Datalog with a strong connection with well-founded semantics \cite{DBLP:journals/tplp/PelovDB07}.
A variety of applications such as planning, diagnosis, etc, have been successfully implemented using it.
In this paper, our goal is to provide an answer to the question  ``\emph{given an answer set $A$ of a program $\Pi$ and an atom $\alpha$, why an atom $\alpha$ is true (or false) in $A$?}''. 

The emergence of XAI has brought significant attention from researchers in ASP community, resulting in numerous proposed systems aimed at addressing this issue such as xclingo{} \cite{Cabalar2020}, \discasp{} \cite{li2021discasp}, \xasp{} \cite{trieu2022explanation}, \expaspc \cite{trieu2021exp}. However, these systems are incapable of handling one or more of the following scenarios: (\emph{i}) false atoms can be explained by the system, (\emph{ii}) the ability to support certain advanced language features.
In this paper, we proposed an improvement system, called \xasptwo, that takes inspiration from the approach used in \xasp{} \cite{trieu2022explanation} and \cite{pontelli2009justifications}. \xasptwo are able to substantially increase scalability and breadth of supported language features while producing explanation graphs with more immediately and consistently useful to users.
The improvement presented here deals with two main issues in explaining the assignment of $\alpha$ in $A$:
(\emph{i}) how to compute a minimum cardinality set of atoms that is assumed to be false such that it is capable of explaining the assignment of $\alpha$ in $A$; and 
(\emph{ii}) how to support sophisticated linguistic constructs such as choice rules and aggregates, which can be involved to explain the falsity of some atoms in easily understandable terms. 

Our main contributions are the following:
\begin{itemize}
\item
A notion of explanation in terms of directed acyclic graphs explains why an atom is (or is not) in an answer set in terms of easy-to-understand inferences originating from a hopefully minimum set of assumed false atoms  (Section~\ref{sec:explanations}). Note that the explanation graph of an atom is restricted to the atoms involved in the relevant rules for the explaining atoms. 

\item 
A proof of existence for the explanations according to the given definition that guarantees the correctness of our implementation (Section~\ref{sec:correctness}).

\item
The implementation of an enhancement system, \xasptwo, for producing explanations powered by ASP and its empirical evaluation(Sections~\ref{sec:meta-programming}--\ref{sec:experiment}).  \xasptwo tackles logic programs with different {\small \tt clingo} constructs such as aggregates and constraints.
\end{itemize}
The supported fragment of ASP includes uninterpreted function symbols, common aggregation functions, comparison expressions, strong negation, constraints, normal rules, and choice rules.
Aggregates are expected to be stratified, to not involve default negation, and to have a single atomic condition.
Choice rules are expected to be unconditional, or otherwise to have exactly one conditional atom with a self-explanatory condition (as for example a range expression or an extensional predicate).
Additionally, to ease the presentation, in Section~\ref{sec:background} we only consider \emph{sum} aggregates, and completely omit uninterpreted function symbols, comparison expressions, strong negation, and conditions in choice rules.
To the best of our knowledge, this is the first explanation generation system that supports different {\small \tt clingo} constructs such as aggregates and constraints.

\section{Background}\label{sec:background}

All sets and sequences considered in this paper are finite.
Let $\mathbf{P}$, $\mathbf{C}$, $\mathbf{V}$ be fixed nonempty sets of \emph{predicate names}, \emph{constants} and \emph{variables}.
Predicates are associated with an \emph{arity}, a non-negative integer.
A \emph{term} is any element in $\mathbf{C} \cup \mathbf{V}$.
An \emph{atom} is of the form $p(\overline{t})$, where $p \in \mathbf{P}$, and $\overline{t}$ is a possibly empty sequence of terms.
A \emph{literal} is an atom possibly preceded by the default negation symbol $\naf\!\!$;
they are referred to as positive and negative literals.

An \emph{aggregate} is of the form
\begin{equation}\label{eq:aggregate}
    \mathit{sum}\{t_a, \overline{t'}: p(\overline{t})\} \odot t_g
\end{equation}
where
$\odot$ is a binary comparison operator, $p \in \mathbf{P}$, $\overline{t}$ and $\overline{t'}$ are possibly empty sequences of terms, and $t_a$ and $t_g$ are terms. 

A \emph{choice} is of the form
\begin{equation}\label{eq:choice}
    t_1 \leq \{\mathit{atoms}\} \leq t_2
\end{equation}
where $\mathit{atoms}$ is a possibly empty sequence of atoms, and $t_1,t_2$ are terms.
Let $\bot$ be syntactic sugar for $1 \leq \{\} \leq 1$.

A \emph{rule} is of the form
\begin{equation}\label{eq:rule}
    \mathit{head} \leftarrow \mathit{body}
\end{equation}
where $\mathit{head}$ is an atom or a choice, and $\mathit{body}$ is a possibly empty sequence of literals and aggregates.
For a rule $r$, let $H(r)$ denote the atom or choice in the head of $r$;
let $B^\Sigma(r)$, $B^+(r)$ and $B^-(r)$ denote the sets of aggregates, positive and negative literals in the body of $r$;
let $B(r)$ denote the set $B^\Sigma(r) \cup B^+(r) \cup B^-(r)$.

A variable $X$ occurring in $B^+(r)$ is a \emph{global variable}.
Other variables occurring among the terms $\overline{t}$ of some aggregate in $B^\Sigma(r)$ of the form \eqref{eq:aggregate} are \emph{local variables}.
And any other variable occurring in $r$ is an \emph{unsafe variable}.
A \emph{safe rule} is a rule with no \emph{unsafe variables}.
A \emph{program} $\Pi$ is a set of safe rules.
Additionally, we assume that aggregates are stratified, that is, the \emph{dependency graph} $\mathcal{G}_\Pi$ having a vertex for each predicate occurring in $\Pi$ and an edge $pq$ whenever there is $r \in \Pi$ with $p$ occurring in $H(r)$ and $q$ occurring in $B^+(r)$ or $B^\Sigma(r)$ is acyclic.

\begin{example}\label{ex:running:1}
Given a connected undirected graph $G$ encoded by predicate $\mathit{edge/2}$, source and sink nodes encoded by predicates $\mathit{source/1}$ and $\mathit{sink/1}$, the following program assigns a direction to each edge so that source nodes can still reach all sink nodes:
\begin{align}
    \label{eq:rule:run:1}
    & 1 \leq \{\mathit{arc}(X,Y);\ \mathit{arc}(Y,X)\} \leq 1 \leftarrow \mathit{edge}(X,Y)\\
    \label{eq:rule:run:2}
    & \mathit{reach}(X,X) \leftarrow \mathit{source}(X)\\
    \label{eq:rule:run:3} 
    & \mathit{reach}(X,Y) \leftarrow \mathit{reach}(X,Z),\ \mathit{arc}(Z,Y)\\
    \label{eq:rule:run:4} 
    & \bot \leftarrow \mathit{source}(X),\ \mathit{sink}(Y),\ \naf \mathit{reach}(X,Y)
\end{align}
If failures on the reachability condition are permitted up to a given threshold encoded by predicate $\mathit{threshold/1}$, the program comprising rules \eqref{eq:rule:run:1}--\eqref{eq:rule:run:3} and
\begin{align}
    \label{eq:rule:run:5}
    & \mathit{fail}(X,Y) \leftarrow \mathit{source}(X),\ \mathit{sink}(Y),\ \naf\mathit{reach}(X,Y)\\
    \label{eq:rule:run:6}
    & \bot \leftarrow \mathit{threshold}(T), \mathit{sum}\{1,X,Y : \mathit{fail}(X,Y)\} > T
\end{align}
can be used.
Note that $X$ and $Y$ are local variables in rule \eqref{eq:rule:run:6}, and all other variables are global.
\hfill$\blacksquare$
\end{example}

A substitution $\sigma$ is a partial function from variables to constants;
the application of $\sigma$ to an expression $E$ is denoted by $E\sigma$.
Let $\mathit{instantiate}(\Pi)$ be the program obtained from rules of $\Pi$ by substituting global variables with constants in $\mathbf{C}$, in all possible ways;
note that local variables are still present in $\mathit{instantiate}(\Pi)$.
The Herbrand base of $\Pi$, denoted $\mathit{base}(\Pi)$, is the set of ground atoms (i.e., atoms with no variables) occurring in $\mathit{instantiate}(\Pi)$.

\begin{figure}[t]
    \centering
    \begin{tikzpicture}
        \node[shape=circle,draw=black,fill=blue,fill opacity=0.7,text=white,text opacity=1] (B) at (0,0) {\bf b};
        \node[shape=circle,draw=black,fill=blue,fill opacity=0.7,text=white,text opacity=1] (A) at (2,0) {\bf a};
        \node[shape=circle,draw=black] (D) at (4,0) {\bf d};
        \node[shape=circle,draw=black,fill=red,fill opacity=0.7,text opacity=1] (C) at (6,0) {\bf c};
    
        \path [-] (A) edge (B);
        \path [-] (A) edge (D);
        \path [-] (D) edge (C);
    \end{tikzpicture}
    \caption{
        The undirected graph used as running example.
        Source vertices in blue, sink vertex in red.
    }\label{fig:running-graph}
\end{figure}
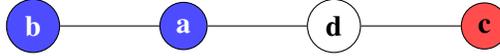

\begin{example}\label{ex:running:2}
Let $\Pi_\mathit{run}$ comprise rules \eqref{eq:rule:run:1}--\eqref{eq:rule:run:3}, \eqref{eq:rule:run:5}--\eqref{eq:rule:run:6} and the facts (i.e., rules with an empty body)
$\mathit{edge}(a,b)$,
$\mathit{edge}(a,d)$,
$\mathit{edge}(d,c)$,
$\mathit{source}(a)$,
$\mathit{source}(b)$,
$\mathit{sink}(c)$, and
$\mathit{threshold}(0)$
(see Figure~\ref{fig:running-graph}).
Hence, $\mathit{instantiate}\\(\Pi_\mathit{run})$ contains, among others, the rules
\begin{align*}
    & 1 \leq \{\mathit{arc}(a,b);\ \mathit{arc}(b,a)\} \leq 1 \leftarrow \mathit{edge}(a,b)\\
    & \bot \leftarrow \mathit{threshold}(0), \mathit{sum}\{1,X,Y : \mathit{fail}(X,Y)\} > 0
\end{align*}
and $\mathit{base}(\Pi_\mathit{run})$ contains
$\mathit{fail}(a,c)$,
$\mathit{fail}(b,c)$, and so on.
\hfill$\blacksquare$
\end{example}

A \emph{(two-valued) interpretation} is a set of ground atoms.
For a two-valued interpretation $I$, relation $I \models \cdot$ is defined as follows:
for a ground atom $p(\overline{c})$, $I \models p(\overline{c})$ if $p(\overline{c}) \in I$, and $I \models \naf p(\overline{c})$ if $p(\overline{c}) \notin I$;
for an aggregate $\alpha$ of the form \eqref{eq:aggregate}, 
the aggregate set of $\alpha$ w.r.t.\ $I$, denoted $\mathit{aggset}(\alpha,I)$, is 
$\{\tuple{t_a,\overline{t'}}\sigma \mid p(\overline{t})\sigma \in I,$ for some substitution $\sigma\}$, and
$I \models \alpha$ if 
$(\sum_{\tuple{c_a,\overline{c'}} \in \mathit{aggset}(\alpha,I)}{c_a}) \odot t_g$ is a true expression over integers;
for a choice $\alpha$ of the form \eqref{eq:choice},
$I \models \alpha$ if \mbox{$t_1 \leq |I \cap \mathit{atoms}| \leq t_2$} is a true expression over integers;
for a rule $r$ with no global variables,
$I \models B(r)$ if $I \models \alpha$ for all $\alpha \in B(r)$, and
$I \models r$ if $I \models H(r)$ whenever $I \models B(r)$;
for a program $\Pi$, $I \models \Pi$ if $I \models r$ for all $r \in \mathit{instantiate}(\Pi)$.

For a rule $r$ of the form \eqref{eq:rule} and an interpretation $I$, let $\mathit{expand}(r,I)$ be the set $\{p(\overline{c}) \leftarrow \mathit{body} \mid p(\overline{c}) \in I$ occurs in $H(r)\}$.
The \emph{reduct} of $\Pi$ w.r.t.\ $I$ is the program comprising the expanded rules of $\mathit{instantiate}(\Pi)$ whose body is true w.r.t.\ $I$, that is,
$\mathit{reduct}(\Pi,I) := \bigcup_{r \in \mathit{instantiate}(Pi),\ I \models B(r)}{\mathit{expand}(r,I)}$.
An \emph{answer set} of $\Pi$ is an interpretation $A$ such that $A \models \Pi$ and no $I \subset A$ satisfies $I \models \mathit{reduct}(\Pi,A)$.

\begin{example}\label{ex:answer-set}
The only answer set $A_\mathit{run}$ of program $\Pi_\mathit{run}$ contains, among others, the atoms
$\mathit{arc}(b,a)$,
$\mathit{arc}(a,d)$,
$\mathit{arc}(d,c)$,
no other instance of $\mathit{arc/2}$,
and no instance of $\mathit{fail/2}$.
Hence, $A_\mathit{run} \models 1 \leq \{\mathit{arc}(a,b); \\ \mathit{arc}(b,a)\} \leq 1$ and
$A_\mathit{run} \not\models \mathit{sum}\{1,X,Y : \mathit{fail}(X,Y)\} > 0$.
\hfill$\blacksquare$
\end{example}

A \emph{three-valued interpretation} is a pair $(L,U)$, where $L,U$ are sets of ground atoms such that $L \subseteq U$;
sets $L$ and $U$, also denoted $(L,U)_1$ and $(L,U)_2$, are the lower and upper bounds on the true atoms, so
atoms in $L$ are true, atoms in $U \setminus L$ are undefined, and all other atoms are false.
The \emph{evaluation function} $\eval{\cdot}{L}{U}$ associates literals and aggregates 
with a truth value among $\mathbf{u}$, $\mathbf{t}$ and $\mathbf{f}$ as follows:
$\eval{\alpha}{L}{U} = \mathbf{u}$ if $\alpha$ is a literal whose atom is $p(\overline{c})$ and $p(\overline{c}) \in U \setminus L$, or
$\alpha$ is an aggregate of the form \eqref{eq:aggregate} and $\mathit{aggset}(\alpha,U \setminus L) \neq \emptyset$, or
$\alpha$ is a choice of the form \eqref{eq:choice} and $(U \setminus L) \cap \mathit{atoms} \neq \emptyset$;
$\eval{\alpha}{L}{U} = \mathbf{t}$ if $\eval{\alpha}{L}{U} \neq \mathbf{u}$ and $L \models \alpha$; and 
$\eval{\alpha}{L}{U} = \mathbf{f}$ if $\eval{\alpha}{L}{U} \neq \mathbf{u}$ and $L \not\models \alpha$.
The evaluation function extends to rule bodies as follows:
$\eval{B(r)}{L}{U} = \mathbf{f}$ if there is $\alpha \in B(r)$ such that $\eval{\alpha}{L}{U} = \mathbf{f}$;
$\eval{B(r)}{L}{U} = \mathbf{t}$ if $\eval{\alpha}{L}{U} = \mathbf{t}$ for all $\alpha \in B(r)$;
otherwise $\eval{B(r)}{L}{U} = \mathbf{u}$.

\begin{example}
For $\alpha$ being $\mathit{sum}\{1,X,Y : \mathit{fail}(X,Y)\} > 0$,
$\eval{\alpha}{\emptyset}{\{\mathit{fail}(a,c)\}} = \mathbf{u}$,
$\eval{\alpha}{\{\mathit{fail}(a,c)\}}{\{\mathit{fail}(a,c)\}} = \mathbf{t}$, and
$\eval{\alpha}{\emptyset}{\emptyset} = \mathbf{f}$.
\hfill$\blacksquare$
\end{example}

Mainstream ASP systems compute answer sets of a given program $\Pi$ by applying several inference rules on (a subset of) $\mathit{instantiate}(\Pi)$, the most relevant ones for this work summarized below.
Let $(L,U)$ be a three-valued interpretation, and $p(\overline{c})$ be a ground atom such that $\eval{p(\overline{c})}{L}{U} = \mathbf{u}$.
Atom $p(\overline{c})$ in $H(r)$ is \emph{inferred true by support} if \mbox{$\eval{B(r)}{L}{U} = \mathbf{t}$}. 
(Actually, if $H(r)$ is a choice of the form \eqref{eq:choice}, inference by support additionally requires that $|\mathit{atoms} \cap U| = t_1$, that is, undefined atoms in $\mathit{atoms} \cap U$ are required to reach the bound $t_1$.
Such extra condition is not relevant for our work, and will not be used, because our explanations aim at associating true atoms with rules with true bodies.)
Atom $p(\overline{c})$ is \emph{inferred false by lack of support} if each rule  $r \in \mathit{instantiate}(\Pi)$ with $p(\overline{c})$ occurring in $H(r)$ is such that $\eval{B(r)}{L}{U} = \mathbf{f}$.
Atom $p(\overline{c})$ is \emph{inferred false by a constraint-like rule} $r \in \mathit{instantiate}(\Pi)$ if $p(\overline{c}) \in B^+(r)$, $\eval{H(r)}{L}{U} = \mathbf{f}$ and $\eval{B(r) \setminus \{p(\overline{c})\}}{L}{U} = \mathbf{t}$.
Atom $p(\overline{c})$ is \emph{inferred false by a choice rule} $r \in \mathit{instantiate}(\Pi)$ if $H(r)$ has the form \eqref{eq:choice}, $p(\overline{c}) \in \mathit{atoms}$, $|\mathit{atoms} \cap L| \geq t_2$ and $\eval{B(r)}{L}{U} = \mathbf{t}$.
Atom $p(\overline{c})$ is \emph{inferred false by well-founded computation} if it belongs to some \emph{unfounded set} $X$ for $\Pi$ w.r.t.\ $(L,U)$, that is, a set $X$ such that
for all rules $r \in \mathit{instantiate}(\Pi)$ at least one of the following conditions holds:
(i) no atom from $X$ occurs in $H(r)$;
(ii) $\eval{B(r)}{L}{U} = \mathbf{f}$;
(iii) $B^+(r) \cap X \neq \emptyset$.

\begin{example}
Given the program $\mathit{instantiate}(\Pi_\mathit{run})$, and the three-valued interpretation $\left(\emptyset,\mathit{base}(\Pi_\mathit{run})\right)$, 
atom $\mathit{edge}(a,a)$ is inferred false by lack of support,
atom $\mathit{source}(a)$ is inferred true by support, and
the set $\{\mathit{edge}(a,a),$ $\mathit{arc}(a,a)\}$ is unfounded.
Given $\left(\{\mathit{arc}(d,c)\},\mathit{base}(\Pi_\mathit{run}) \setminus\{\mathit{reach}(a,c)\}\right)$, atom $\mathit{reach}(a,d)$ is inferred false by the constraint-like rule \eqref{eq:rule:run:3}, and
$\mathit{arc}(c,d)$ is inferred false by the choice rule \eqref{eq:rule:run:1}.
\hfill$\blacksquare$
\end{example}

\section{Explanations}\label{sec:explanations}
Let $\Pi$ be a program, and $A$ be one of its answer sets.
A \emph{well-founded derivation} for $\Pi$ w.r.t.\ $A$, denoted $\mathit{wf}(\Pi,A)$, is obtained from the interpretation $(\emptyset,\mathit{base}(\Pi))$ by iteratively 
(i) adding to its lower bound atoms of $A$ that are inferred true by support, and 
(ii) removing from its upper bound atoms belonging to some unfounded set.
Note that $\mathit{wf}(\Pi,A)$ is computed as a preprocessing step. 

\begin{example}
Given $\Pi_\mathit{run}$ and $A_\mathit{run}$ from Examples~\ref{ex:running:2}--\ref{ex:answer-set},
the lower bound of $\mathit{wf}(\Pi_\mathit{run},A_\mathit{run})$ contains head atoms in Example~\ref{ex:running:2},
$\mathit{arc}(b,a)$,
$\mathit{arc}(a,d)$,
$\mathit{arc}(d,c)$,
$\mathit{reach}(a,a)$,
$\mathit{reach}(b,b)$,
$\mathit{reach}(a,d)$,
$\mathit{reach}(a,c)$,
$\mathit{reach}(b,a)$,
$\mathit{reach}(b,c)$, and
$\mathit{reach}(b,d)$.
The upper bound 
additionally contains
$\mathit{arc}(a,b)$,
$\mathit{arc}(d,a)$,
$\mathit{arc}(c,d)$,
and several instances of $\mathit{reach/2}$ and $\mathit{fail/2}$.
\hfill$\blacksquare$
\end{example}

An \emph{explaining derivation} for $\Pi$ and $A$ from $(L,U)$ is obtained by iteratively
(i) adding to $L$ atoms of $A$ that are inferred true by support, and
(ii) removing from $U$ atoms that are inferred false by lack of support, constraint-like rules and choice rules.
An \emph{assumption set} for $\Pi$ and $A$ is a set $X \subseteq \mathit{base}(\Pi) \setminus A$ of ground atoms such that
the explaining derivation for $\Pi$ and $A$ from $(\emptyset, \mathit{wf}(\Pi,A)_2 \setminus X)$ terminates with $A$ (in words, $A$ is reconstructed from the false atoms of the well-founded derivation extended with $X$).
Let $\mathit{AS}(\Pi,A)$ be the set of assumption sets for $\Pi$ and $A$.
A \emph{minimal assumption set} for $\Pi$, $A$ and a ground atom $\alpha$ is a set $X \in \mathit{AS}(\Pi,A)$ such that $X' \subset X$ implies $X' \notin \mathit{AS}(\Pi,A)$, and $\alpha \in X$ implies $\alpha \in X'$ for all $X' \in \mathit{AS}(\Pi,A)$.
(In other words, we prefer assumption sets not including the atom to explain.
When all assumption sets include the atom to explain, we opt for the singleton comprising the atom to explain alone.)
Let $\mathit{MAS}(\Pi,A,\alpha)$ be the set of minimal assumption sets for $\Pi$, $A$ and $\alpha$.

\begin{example}
Set $\mathit{base}(\Pi_\mathit{run}) \setminus A_\mathit{run}$ is an assumption set for $\Pi_\mathit{run}$ and its answer set $A_\mathit{run}$.
It can be checked that also $\emptyset \in \mathit{AS}(\Pi_\mathit{run},A_\mathit{run},\alpha)$, and it is indeed the only minimal assumption set in this case, for any atom in $\mathit{base}(\Pi_\mathit{run})$.
\hfill$\blacksquare$
\end{example}

Given an assumption set $X$ and an explaining derivation from $(\emptyset, \mathit{wf}(\Pi,A)_2 \setminus X)$, a directed acyclic graph (DAG) can be obtained as follows:
The vertices of the graph are 
the atoms in $\mathit{base}(\Pi)$ and the aggregates occurring in $\mathit{instantiate}(\Pi)$.
(The vertex $p(\overline{c})$ is also referred to as $\naf p(\overline{c})$.)
Any aggregate of the form \eqref{eq:aggregate} is linked to instances of $\mathit{p(\overline{t})}$.
Atoms inferred true by support due to a rule $r \in \mathit{instantiate}(\Pi)$ 
are linked to elements of $B(r)$.
Any atom $\alpha$ inferred false by lack of support is linked to an element of $B(r)$ that is inferred false before $\alpha$, for each rule $r \in \mathit{instantiate}(\Pi)$ such that $\alpha$ occurs in $H(r)$.
Any atom $\alpha$ inferred false by a constraint-like rule $r \in \mathit{instantiate}(\Pi)$ is linked to the atoms occurring in $H(r)$ and the elements of $B(r) \setminus \{\alpha\}$.
Any atom $\alpha$ inferred false by a choice rule $r \in \mathit{instantiate}(\Pi)$ is linked to the atoms occurring in $H(r)$ that are true in $A$, and to the elements of $B(r)$.
A portion of an example DAG is reported in Figure~\ref{fig:dag-run}.

\begin{figure}
    \centering
    \includegraphics[width=.6\linewidth]{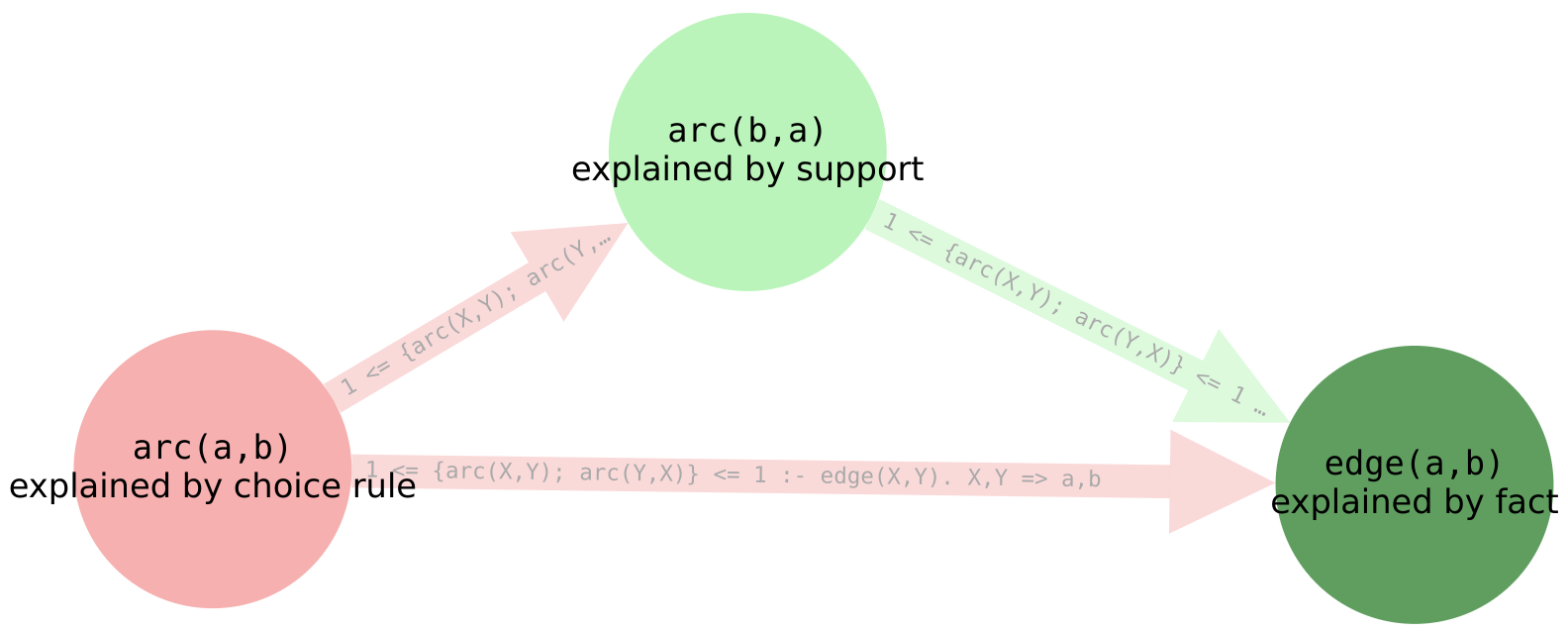}
    \caption{
        Induced DAG on the vertices reachable from $\mathit{arc(a,b)}$ for the minimal assumption set $\emptyset$ for $\Pi_\mathit{run}$.
    }\label{fig:dag-run}
\end{figure}

\section{Existence of Minimal Assumption Sets}\label{sec:correctness}

This section is devoted to formally show that the existence of minimal assumption sets is guaranteed, and so are DAGs as defined in Section~\ref{sec:explanations}.

\begin{restatable}[Main Theorem]{theorem}{MainTheorem}\label{thm:existence}
Let $\Pi$ be a program, $A$ one of its answer sets, and $\alpha$ a ground atom in $\mathit{base}(\Pi)$.
Set $\mathit{MAS}(\Pi,A,\alpha)$ is nonempty.
\end{restatable}

To prove the above theorem, we introduce some additional notation and claims.
Let $\Pi$ be a program, and $(L,U)$ be a three-valued interpretation.
We denote by $\Pi,L,U \vdash \alpha$ the fact that $\alpha \in \mathit{base}(\Pi)$ is inferred true by support, which is the case when $\eval{\alpha}{L}{U} = \mathbf{u}$, and there is $r \in \mathit{instantiate}(\Pi)$ such that $\alpha$ occurs in $H(r)$ and $\eval{B(r)}{L}{U} = \mathbf{t}$, as defined in Section~\ref{sec:background}.
Similarly, we denote by $\Pi,L,U \vdash \naf \alpha$ the fact that $\alpha \in \mathit{base}(\Pi)$ is inferred false by lack of support, constraint-like rules and choice rules, which is the case when $\eval{\alpha}{L}{U} = \mathbf{u}$, and one of the following conditions holds:
each rule $r \in \mathit{instantiate}(\Pi)$ with $\alpha$ occurring in $H(r)$ is such that $\eval{B(r)}{L}{U} = \mathbf{f}$;
there is $r \in \mathit{instantiate}(\Pi)$ with $\alpha \in B^+(r)$, $\eval{H(r)}{L}{U} = \mathbf{f}$ and $\eval{B(r) \setminus \{\alpha\}}{L}{U} = \mathbf{t}$;
there is $r \in \mathit{instantiate}(\Pi)$ with $H(r)$ of the form \eqref{eq:choice}, $\alpha \in \mathit{atoms}$, $|\mathit{atoms} \cap L| \geq t_2$ and $\eval{B(r)}{L}{U} = \mathbf{t}$.

The \emph{explaining derivation operator} $D_\Pi$ is defined as
\begin{equation*}
    \begin{array}{r}
        D_\Pi(L,U) := (L \cup \{\alpha \in \mathit{base}(\Pi) \mid \Pi,L,U \vdash \alpha\}, \\
        U \setminus \{\alpha \in \mathit{base}(\Pi) \mid \Pi,L,U \vdash \naf \alpha\}).
    \end{array}
\end{equation*}
Let $(L,U) \sqsubseteq (L',U')$ denote the fact that $L \subseteq L' \subseteq U' \subseteq U$, i.e., everything that is true w.r.t.\ $(L,U)$ is true w.r.t.\ $(L',U')$, and everything that is false w.r.t.\ $(L,U)$ is false w.r.t.\ $(L',U')$.

\begin{restatable}{lemma}{LemExplainingDerivationMonotonic}\label{lem:explaning-derivation-monotonic}
Operator $D_\Pi$ is monotonic w.r.t.\ ${}\sqsubseteq{}$.
\end{restatable}

 \begin{proof}
For $(L,U) \sqsubseteq (L',U')$, we shall show that $D_\Pi(L,U) \sqsubseteq D_\Pi(L',U')$ holds.
For $\alpha \in D_\Pi(L,U)_1 \setminus L$ such that $\alpha \notin L'$, we have $\Pi,L,U \vdash \alpha$, that is, there is $r \in \mathit{instantiate}(\Pi)$ such that $\alpha$ occurs in $H(r)$ and $\eval{B(r)}{L}{U} = \mathbf{t}$.
As $(L,U) \sqsubseteq (L',U')$, we have that $\eval{B(r)}{L'}{U'} = \mathbf{t}$, that is, $\Pi,L',U' \vdash \alpha$ holds, and therefore $\alpha \in D_\Pi(L',U')_1 \setminus L$.

For $\alpha \in U \setminus D_\Pi(L,U)_2$ such that $\alpha \in U'$, we have $\Pi,L,U \vdash \naf \alpha$, and therefore we have three cases:
\begin{enumerate}
\item 
Each rule $r \in \mathit{instantiate}(\Pi)$ with $\alpha$ occurring in $H(r)$ is such that $\eval{B(r)}{L}{U} = \mathbf{f}$.
As $(L,U) \sqsubseteq (L',U')$, $\eval{B(r)}{L'}{U'} = \mathbf{f}$ holds.

\item
There is $r \in \mathit{instantiate}(\Pi)$ with $\alpha \in B^+(r)$, $\eval{H(r)}{L}{U} = \mathbf{f}$ and $\eval{B(r) \setminus \{\alpha\}}{L}{U} = \mathbf{t}$.
As $(L,U) \sqsubseteq (L',U')$, $\eval{H(r)}{L'}{U'} = \mathbf{f}$ and $\eval{B(r) \setminus \{\alpha\}}{L'}{U'} = \mathbf{t}$.

\item
There is $r \in \mathit{instantiate}(\Pi)$ with $H(r)$ of the form \eqref{eq:choice}, $\alpha \in \mathit{atoms}$, $|\mathit{atoms} \cap L| \geq t_2$ and $\eval{B(r)}{L}{U} = \mathbf{t}$.
As $(L,U) \sqsubseteq (L',U')$, $|\mathit{atoms} \cap L'| \geq t_2$ and $\eval{B(r)}{L'}{U'} = \mathbf{t}$.
\end{enumerate}
In any case, $\Pi,L',U' \vdash \alpha$ holds, and therefore $\alpha \in U' \setminus D_\Pi(L',U')_2$.
\end{proof}

\begin{restatable}{lemma}{LemApproacingAnswerSet}\label{lem:approacing-answer-set}
$L \subseteq A \subseteq U$ implies $D_\Pi(L,U)_1 \subseteq A \subseteq D_\Pi(L,U)_2$.
\end{restatable}

\begin{proof}
For $\alpha \in D_\Pi(L,U)_1 \setminus L$ we have $\Pi,L,U \vdash \alpha$, that is, there is $r \in \mathit{instantiate}(\Pi)$ such that $\eval{B(r)}{L}{U} = \mathbf{t}$.
Hence, $A \models B(r)$, and therefore $\mathit{expand}(r,A) \subseteq \mathit{reduct}(\Pi,A)$.
In particular, $\alpha \leftarrow B(r)$ belongs to the reduct, and therefore $\alpha \in A$.

For $\alpha \in U \setminus D_\Pi(L,U)_2$ we have $\Pi,L,U \vdash \naf \alpha$ and we have to show that $\alpha \notin A$.
Three cases:
\begin{enumerate}
\item 
Each rule $r \in \mathit{instantiate}(\Pi)$ with $\alpha$ occurring in $H(r)$ is such that $\eval{B(r)}{L}{U} = \mathbf{f}$.

\item
There is $r \in \mathit{instantiate}(\Pi)$ with $\alpha \in B^+(r)$, $\eval{H(r)}{L}{U} = \mathbf{f}$ and $\eval{B(r) \setminus \{\alpha\}}{L}{U} = \mathbf{t}$.

\item
There is $r \in \mathit{instantiate}(\Pi)$ with $H(r)$ of the form \eqref{eq:choice}, $\alpha \in \mathit{atoms}$, $|\mathit{atoms} \cap L| \geq t_2$ and $\eval{B(r)}{L}{U} = \mathbf{t}$.
\end{enumerate}
In the first case, $A \setminus \{\alpha\} \models \mathit{reduct}(\Pi,A)$, and therefore $A \setminus \{\alpha\} = A$ because $A$ is an answer set of $\Pi$.
In the other two cases, $\alpha \notin A$ because $A \models \Pi$ by assumption.
\end{proof}

The explaining derivation from $(L,U)$ is obtained as the fix point of the sequence $(L_0,U_0) := (L,U)$, $(L_{i+1},U_{i+1}) := D_\Pi(L_i,U_i)$ for $i \geq 0$.
Note that the fix point is reached in at most $|\mathit{base}(\Pi)|$ steps because of Lemma~\ref{lem:explaning-derivation-monotonic} and each application of $D_\Pi$ reduces the undefined atoms (or is a fix point). Thus, the system eventually terminates in at most $|\mathit{base}(\Pi)|$  steps.

\begin{restatable}{lemma}{LemBigUS}\label{lem:big-us}
For any answer set $A$ of $\Pi$,
set $\mathit{base}(\Pi) \setminus A$ is an assumption set for $\Pi$ and $A$.
\end{restatable}

\begin{proof}
Let $(L,U)$ be the explaining derivation from $(\emptyset, \mathit{base}(\Pi) \setminus A)$.
Thanks to Lemma~\ref{lem:approacing-answer-set}, it is sufficient to show that $p(\overline{c}) \in A$ implies $p(\overline{c}) \in L$.
Let us consider a topological ordering $C_1,\ldots,C_n$ ($n \geq 1$) for the strongly connected components of $\mathcal{G}_\Pi$, and let $p \in C_i$.
We use induction on $i$.
Since $p(\overline{c}) \in A$, there must be $r \in \mathit{reduct}(\Pi,A)$ such that $H(r) = p(\overline{c})$ and $A \models B(r)$.
Hence, $\eval{B^-(r)}{L}{U} = \mathbf{t}$.
Moreover, $\eval{B^\Sigma(r)}{L}{U} = \mathbf{t}$, either because $i = 1$ and $B^\Sigma(r) = \emptyset$, or because of the induction hypothesis.
Therefore, to have $\alpha \notin L$, it must be the case that $\eval{B^+(r)}{L}{U} \neq \mathbf{t}$ for all such rules, but in this case $L \models \mathit{reduct}(\Pi,A)$, a contradiction with the assumption that $A$ is an answer set of $\Pi$.
\end{proof}

Given Lemma~\ref{lem:big-us}, the proof of Main Theorem is immediate by the definition of $\mathit{MAS}(\Pi,A,\alpha)$ as following:

\begin{proof} [Proof of Main Theorem]
By definition, a minimal assumption set for $\Pi$, $A$ and $\alpha$ is a set $X \in \mathit{AS}(\Pi,A)$ such that $X' \subset X$ implies $X' \notin \mathit{AS}(\Pi,A)$, and $\alpha \in X$ implies $\alpha \in X'$ for all $X' \in \mathit{AS}(\Pi,A)$.
Lemma~\ref{lem:big-us} guarantees the existence of an assumption set for $\Pi$ and $A$.
Existence of a minimal assumption set for $\Pi$, $A$ and $\alpha$ is therefore guaranteed.
\end{proof}

\section{Generation via Meta-Programming}\label{sec:meta-programming} 
By leveraging ASP systems, the concepts introduced in Section~\ref{sec:explanations} can be computed.
A meta-programming approach is presented in this section, where the full language of ASP is used, including constructs omitted in the previous sections, like weak constraints, uninterpreted functions, conditional literals and @-terms.
The reader is referred to \cite{DBLP:journals/tplp/CalimeriFGIKKLM20} for details.
We will use the name \emph{ASP programs} for encodings using the full language of ASP, in contrast to the name \emph{program} that we use for encodings using the restricted syntax introduced in Section~\ref{sec:background}.

Program $\Pi$, answer set $A$ and the atom to explain are encoded by a set of facts obtained by computing the unique answer set of the ASP program $\mathit{serialize}(\Pi,A,\alpha)$, defined next.
Each atom $p(\overline{c})$ in $\mathit{base}(\Pi)$ is encoded by a fact \lstinline|atom($p(\overline{c})$)|;
moreover, the encoding includes a fact \lstinline|true($p(\overline{c})$)| if $p(\overline{c}) \in A$, and \lstinline|false($p(\overline{c})$)| otherwise;
additionally, if $p(\overline{c})$ is false in $\mathit{wf}(\Pi,A)$, the encoding includes a fact \lstinline|explained$\_$by($p(\overline{c})$, initial$\_$well$\_$founded)|.
As for $\alpha$, the encoding includes a fact \lstinline|explain($\alpha$)|.
Each rule $r$ of $\mathit{instantiate}(\Pi)$ is encoded by
\begin{asp}
rule($\mathit{id}(\overline{X})$) :- atom($p_1(\overline{t_1})$), ..., atom($p_n(\overline{t_n})$).
\end{asp}
where $\mathit{id}$ is an identifier for $r$, $\overline{X}$ are the global variables of $r$, and $B^+(r) = \{p_i(\overline{t_i}) \mid i = 1, \ldots, n\}$;
moreover, the encoding includes
\begin{asp}
  head($\mathit{id}(\overline{X})$,$p(\overline{t})$) :- rule($\mathit{id}(\overline{X})$).
  pos_body($\mathit{id}(\overline{X})$,$p'(\overline{t'})$) :- rule($\mathit{id}(\overline{X})$).
  neg_body($\mathit{id}(\overline{X})$,$p''(\overline{t''})$) :- rule($\mathit{id}(\overline{X})$).
\end{asp}
for each $p(\overline{t})$ occurring in $H(r)$, $p'(\overline{t'}) \in B^+(r)$ and $p''(\overline{t''}) \in B^(r)$;
additionally, for each aggregate $\alpha$ of the form \eqref{eq:aggregate} in $B^\Sigma(r)$, the encoding includes
\begin{asp}
pos_body($\mathit{id}(\overline{X})$,$\mathit{agg}(\overline{X})$) :- rule($\mathit{id}(\overline{X})$).
aggregate($\mathit{agg}(\overline{X})$) :- rule($\mathit{id}(\overline{X})$).
true($\mathit{agg}(\overline{X})$) :- rule($\mathit{id}(\overline{X})$), 
    #sum$\{t_a,\overline{t'}$ : true($p(\overline{t})$)$\} \odot t_g$.
false($\mathit{agg}(\overline{X})$):- rule($\mathit{id}(\overline{X})$), not true($\mathit{agg}(\overline{X})$).

rule($\mathit{agg}(\overline{X})$):- aggregate($\mathit{agg}(\overline{X})$),true($\mathit{agg}(\overline{X})$).
head($\mathit{agg}(\overline{X})$,$\mathit{agg}(\overline{X})$) :- rule($\mathit{agg}(\overline{X})$).
pos_body($\mathit{agg}(\overline{X})$,$p(\overline{t})$):- rule($\mathit{agg}(\overline{X})$),true($p(\overline{t})$).
neg_body($\mathit{agg}(\overline{X})$,$p(\overline{t})$):-rule($\mathit{agg}(\overline{X})$),false($p(\overline{t})$).

rule(($\mathit{agg}(\overline{X})$,$p(\overline{t})$)) :- aggregate($\mathit{agg}(\overline{X})$), false($\mathit{agg}(\overline{X})$), atom($p(\overline{t})$).
head(($\mathit{agg}(\overline{X})$,$p(\overline{t})$),$\mathit{agg}(\overline{X})$):- rule(($\mathit{agg}(\overline{X})$,$p(\overline{t})$)).
pos_body(($\mathit{agg}(\overline{X})$,$p(\overline{t})$),$p(\overline{t})$) :- rule(($\mathit{agg}(\overline{X})$,$p(\overline{t})$)), false($p(\overline{t})$).
neg_body(($\mathit{agg}(\overline{X})$,$p(\overline{t})$),$p(\overline{t})$) :- rule(($\mathit{agg}(\overline{X})$,$p(\overline{t})$)), true($p(\overline{t})$).
\end{asp}
where $\mathit{agg}$ is an identifier for $\alpha$;
finally, if $H(r)$ is a choice of the form \eqref{eq:choice}, the encoding includes
\begin{asp}
  choice($\mathit{id}(\overline{X})$,$t_1$,$t_2$) :- rule($\mathit{id}(\overline{X})$).
\end{asp}
Note that a true ground aggregate of the form \eqref{eq:aggregate} identified by $\mathit{agg}(\overline{c})$ is associated with a single rule whose body becomes true after all instances of $p(\overline{t})$ are assigned the truth value they have in the answer set $A$;
on the other hand, a false aggregate is associated with one rule for each instance of $p(\overline{t})$, whose bodies becomes false when instances of $p(\overline{t})$ are assigned the truth value they have in the answer set $A$.

\begin{example}
Recall $\Pi_\mathit{run}$ and $A_\mathit{run}$ from Examples~\ref{ex:running:2}--\ref{ex:answer-set}.
The ASP program $\mathit{serialize}(\Pi_\mathit{run},A,\mathit{arc(a,b)})$ includes
{\rm
\begin{asp}
atom(edge(a,b)).  atom(arc(b,a)).  atom (arc(a,b)).  explain(arc(a,b)).
true(edge(a,b)).  true(arc(b,a)).  false(arc(a,b)).

rule(r4(X,Y)) :- atom(edge(X,Y)).
choice(r4(X,Y),1,1) :- rule(r4(X,Y)).
head(r4(X,Y), arc(X,Y)) :- rule(r4(X,Y)).
head(r4(X,Y), arc(Y,X)) :- rule(r4(X,Y)).
pos_body(r4(X,Y), edge(X,Y)):- rule(r4(X,Y)).

aggregate(agg1(T)) :- rule(r9(T)).
true(agg1(T)) :- rule(r9(T)), #sum{1,X,Y : true(fail(X,Y))} > T.
\end{asp}
}\noindent
and several other rules.
The answer set of $\mathit{serialize}(\Pi_\mathit{run},A,\mathit{arc(a,b)})$ includes, among other atoms, {\rm\lstinline|aggregate(agg1(0))|} and {\rm\lstinline|false(agg1(0))|}.
\hfill$\blacksquare$
\end{example}

\begin{figure}[t]
\begin{lstlisting}[literate={~} {$\sim$}{1}, basicstyle=\ttfamily\scriptsize]
{assume_false(Atom)} :- false(Atom), not aggregate(Atom).
:~ false(Atom), assume_false(Atom), not explain(Atom). [1@1, Atom]
:~ false(Atom), assume_false(Atom),     explain(Atom). [1@2, Atom]

has_explanation(Atom) :- explained_by(Atom,_).
:- atom(X), #count{Reason: explained_by(Atom,Reason)} != 1.

explained_by(Atom, assumption) :- assume_false(Atom).

{explained_by(Atom, (support, Rule))} :- head(Rule,Atom), true(Atom);
    true (BAtom) : pos_body(Rule,BAtom);  has_explanation(BAtom) : pos_body(Rule,BAtom);
    false(BAtom) : neg_body(Rule,BAtom);  has_explanation(BAtom) : neg_body(Rule,BAtom).

{explained_by(Atom, lack_of_support)} :- false(Atom);  false_body(Rule) : head(Rule,Atom).
false_body(Rule) :- rule(Rule);  pos_body(Rule,BAtom), false(BAtom), has_explanation(BAtom).
false_body(Rule) :- rule(Rule);  neg_body(Rule,BAtom), true(BAtom), has_explanation(BAtom).

{explained_by(Atom, (required_to_falsify_body, Rule))} :- false(Atom), not aggregate(Atom);
    pos_body(Rule,Atom), false_head(Rule); true(BAtom) : pos_body(Rule,BAtom), BAtom != Atom;
    has_explanation(BAtom) : pos_body(Rule,BAtom), BAtom != Atom;
    false(BAtom) : neg_body(Rule,BAtom);  has_explanation(BAtom) : neg_body(Rule,BAtom).
explained_head(Rule) :- rule(Rule); has_explanation(HAtom) : head(Rule,HAtom).
false_head(Rule) :- explained_head(Rule), not choice(Rule,_,_);
    false(HAtom) : head(Rule,HAtom).
false_head(Rule) :- explained_head(Rule), choice(Rule, LowerBound, UpperBound);
    not LowerBound <= #count{HAtom' : head(Rule,HAtom'), true(HAtom')} <= UpperBound.

{explained_by(Atom, (choice_rule, Rule))} :- false(Atom);
    head(Rule,Atom), choice(Rule, LowerBound, UpperBound);
    true(BAtom) : pos_body(Rule,BAtom);  has_explanation(BAtom) : pos_body(Rule,BAtom);
    false(BAtom) : neg_body(Rule,BAtom);  has_explanation(BAtom) : neg_body(Rule,BAtom);
    #count{HAtom : head(Rule, HAtom), true(HAtom), has_explanation(HAtom)} = UpperBound.
\end{lstlisting}
    \caption{ASP program $\Pi_\mathit{MAS}$ for computing a minimal assumption set}\label{fig:mas}
\end{figure}

\begin{figure}[t]
\begin{lstlisting}[basicstyle=\ttfamily\scriptsize]
link(Atom, BAtom) :- explained_by(_, Atom, (support, Rule));  pos_body(Rule, BAtom).
link(Atom, BAtom) :- explained_by(_, Atom, (support, Rule));  neg_body(Rule, BAtom).

{link(Atom, A) : pos_body(Rule,A), false(A), explained_by(I,A,_), I < Index;
 link(Atom, A) : neg_body(Rule,A), true (A), explained_by(I,A,_), I < Index} = 1 :- 
    explained_by(_, Atom, lack_of_support);  head(Rule, Atom).

link(Atom,A) :- explained_by(_, Atom, (required_to_falsify_body, Rule)); head(Rule,A).
link(At,A) :- explained_by(_,At,(required_to_falsify_body, Rule)); pos_body(Rule,A), A!=At.
link(Atom,A) :- explained_by(_,Atom,(required_to_falsify_body, Rule)); neg_body(Rule,A).

link(Atom, HAtom) :- explained_by(_,Atom,(choice_rule, Rule)); head(Rule,HAtom), true(HAtom).
link(Atom, BAtom) :- explained_by(_, Atom, (choice_rule, Rule)); pos_body(Rule, BAtom).
link(Atom, BAtom) :- explained_by(_, Atom, (choice_rule, Rule)); neg_body(Rule, BAtom).
\end{lstlisting}

    \caption{ASP program $\Pi_\mathit{DAG}$ for computing a directed acyclic graph associated with an explaining derivation}\label{fig:dag}
\end{figure}

The ASP program $\Pi_\mathit{MAS}$ reported in Figure~\ref{fig:mas}, coupled with a fact for each atom in the answer set of $\mathit{serialize(\Pi,A,\alpha)}$, has optimal answer sets corresponding to cardinality-minimal elements in $\mathit{MAS}(\Pi,A,\alpha)$.
Intuitively, line~1 guesses the assumption set, line~2--3 minimizes the size of the assumption set (preferring to not assume the falsity of the atom to explain), and lines~4--5 impose that each atom must have exactly one explanation.
The other rules encode the explaining derivation for $\Pi$ and $A$ from $\mathit{wf}(\Pi,A) \setminus X$, where $X$ is the guessed assumption set.

Given a minimal assumption set encoded by predicate \lstinline|assume$\_$false/1|, an explaining derivation can be computed by removing lines~1--3 from the ASP program $\Pi_\mathit{MAS}$.
Let $\Pi_\mathit{EXP}$ be such an ASP program.
Finally, given an explaining derivation encoded by \lstinline|explained$\_$by(Index,Atom,Reason)|, with the additional \lstinline|Index| argument encoding the order in the sequence,
a DAG linking atoms according to the derivation can be computed by the ASP program $\Pi_\mathit{DAG}$ reported in Figure~\ref{fig:dag}.

\begin{example}
Let $\Pi_S$ have a fact for each atom in the answer set of $\mathit{serialize}(\Pi_\mathit{run},A_\mathit{run},\mathit{arc}(a,b))$.
$\Pi_\mathit{MAS} \cup \Pi_S$ generates the empty assumption set.
$\Pi_\mathit{EXP} \cup \Pi_S \cup \emptyset$ generates an explaining derivation, for example one including
$\mathit{explained\_by}(\mathit{edge}(a,b),(\mathit{support},r6))$,
$\mathit{explained\_by}(\mathit{arc}(b,a),(\mathit{support},r1(a,b)))$ and
$\mathit{explained\_by}(\mathit{arc}(a,b),(\mathit{choice\_rule},r1(a,b)))$.
Let $\Pi_E$ have a fact for each instance of $\mathit{explained\_by/3}$ in the explaining derivation.
$\Pi_\mathit{DAG} \cup \Pi_S \cup \Pi_E$ generates a DAG, for example one including
$\mathit{link}(\mathit{arc}(b,a),\\ \mathit{edge}(a,b))$,
$\mathit{link}(\mathit{arc}(a,b),\mathit{arc}(b,a))$ and
$\mathit{link}(\mathit{arc}(a,b),\mathit{edge}(a,b))$.
\hfill$\blacksquare$
\end{example}

\section{Implementation and Experiment}\label{sec:experiment}
We deployed an XAI system for ASP named \xasptwo, which is powered by the \system{clingo\ python\ api} \cite{kaminskiromeroschaubwanko2023}. 
By taking an ASP program $\Pi$, one of its answer sets $A$, and an atom $\alpha$ as input, \xasptwo is capable of producing minimal assumption sets, explaining derivations, and DAGs as output to assist the user in determining the assignment of $\alpha$.
The source code is available at \url{https://github.com/alviano/xasp} and an example DAG is given at \url{https://xasp-navigator.netlify.app/}.

The pipeline implemented by \xasptwo starts with the serialization of the input data, which is obtained by means of an ASP program crafted from the abstract syntax tree of $\Pi$ and whose answer set identifies the relevant portion of $\mathit{instantiate}(\Pi)$ and $\mathit{base}(\Pi)$.
In a nutshell, ground atoms provided by the user, $A \cup \{\alpha\}$, are part of $\mathit{base}(\Pi)$ and used to instantiate rules of $\Pi$ (by matching positive body literals), which in turn may extend $\mathit{base}(\Pi)$ with other ground atoms occurring in the instantiated rules;
possibly, some atoms of $\mathit{base}(\Pi)$ of particular interest can be explicitly provided by the user.
Aggregates are also processed automatically by means of an ASP program, and so is the computation of false atoms in the well-founded derivation $\mathit{wf}(\Pi,A)$.

Obtained $\mathit{serialize}(\Pi,A,\alpha)$, \xasptwo proceeds essentially as described in Section~\ref{sec:meta-programming}, by computing a minimal assumption set, an explaining derivation and an explanation DAG.
As an additional optimization, the explaining derivation is shrunk to the atoms reachable from $\alpha$, utilizing an ASP program.
Finally, the user can opt for a few additional steps:
obtain a graphical representation by means of the \system{igraph} network analysis package (\url{https://igraph.org/});
obtain an interactive representation in \url{https://xasp-navigator.netlify.app/};
ask for different minimal assumption sets, explaining derivations and DAGs.

We assessed \xasptwo empirically on the commercial application. 
The ASP program comprises 420 rules and 651 facts.
After grounding, there are 4261 ground rules and 4468 ground atoms.
The program was expected to have a unique answer set, but two answer sets were actually computed.
Our experiment was run on an Intel Core i7-1165G7 @2.80 GHz and 16 GB of RAM.
\xasptwo computed a DAG for the unexpected true atom, \lstinline|behaves$\_$inertially(testing$\_$posTestNeg,121)|, in 14.85 seconds on average, over 10 executions.
The DAG comprises 87 links, 45 internal nodes and 20 leaves, only one of which is explained by assumption;
only 30 of the 420 symbolic rules and 11 of the 651 facts are involved in the DAG;
at the ground level, only 48 of the 4261 ground rules and 65 of the 4468 ground atoms are involved.
Additionally, we repeated the experiment on 10 randomly selected atoms with respect to two different answer sets, repeating each test case 10 times.
We measured an average runtime of $14.79$ seconds, with a variance of $0.004$ seconds.

\begin{table}[!th]
    \centering
  \caption{The action preconditions and effects in Blockworlds problem}
  \label{table:actiondecription}
  \begin{tabular}{p{0.2\textwidth}p{0.35\textwidth}p{0.35\textwidth}} \hline
    \textbf{Action} & \textbf{Precondition} & \textbf{Effects} \\ \hline
    $stack(X,Y)$ \newline - stack block $X$ is on block $Y$ &     
            Block $Y$ is clear
            \newline  The agent holds the block $X$  &
            $X$ is clear
            \newline $X$ is on $Y$
            \newline $Y$ is no longer clear
            \newline The agent does not hold anything     \\ \hline

    $unstack(X,Y)$ \newline - unstack block $X$ is on block $Y$    & 
            $X$ is clear
            \newline $X$ is on $Y$
            \newline The agent does not hold anything  &  
            The agent holds the block $X$
            \newline $Y$ becomes clear
            \newline $X$ is not clear     \\ \hline

    $pickup(X)$ \newline - pickup block $X$ from the table   & 
            $X$ is clear
            \newline $X$ is on the table
            \newline the agent does not hold anything
             &  
            The agent holds the block $X$
            \newline $X$ is no longer on the table and is not clear
                \\ \hline

    $putdown(X)$ \newline - put down block $X$ onto the table   & 
            The agent holds the block $X$
             &  
            $X$ is clear
            \newline $X$ is on the table
            \newline the agent does not hold anything
                \\ \hline
  \end{tabular}
\end{table}

\begin{figure}[!th]
    \centering
    \includegraphics[width=.4\linewidth]{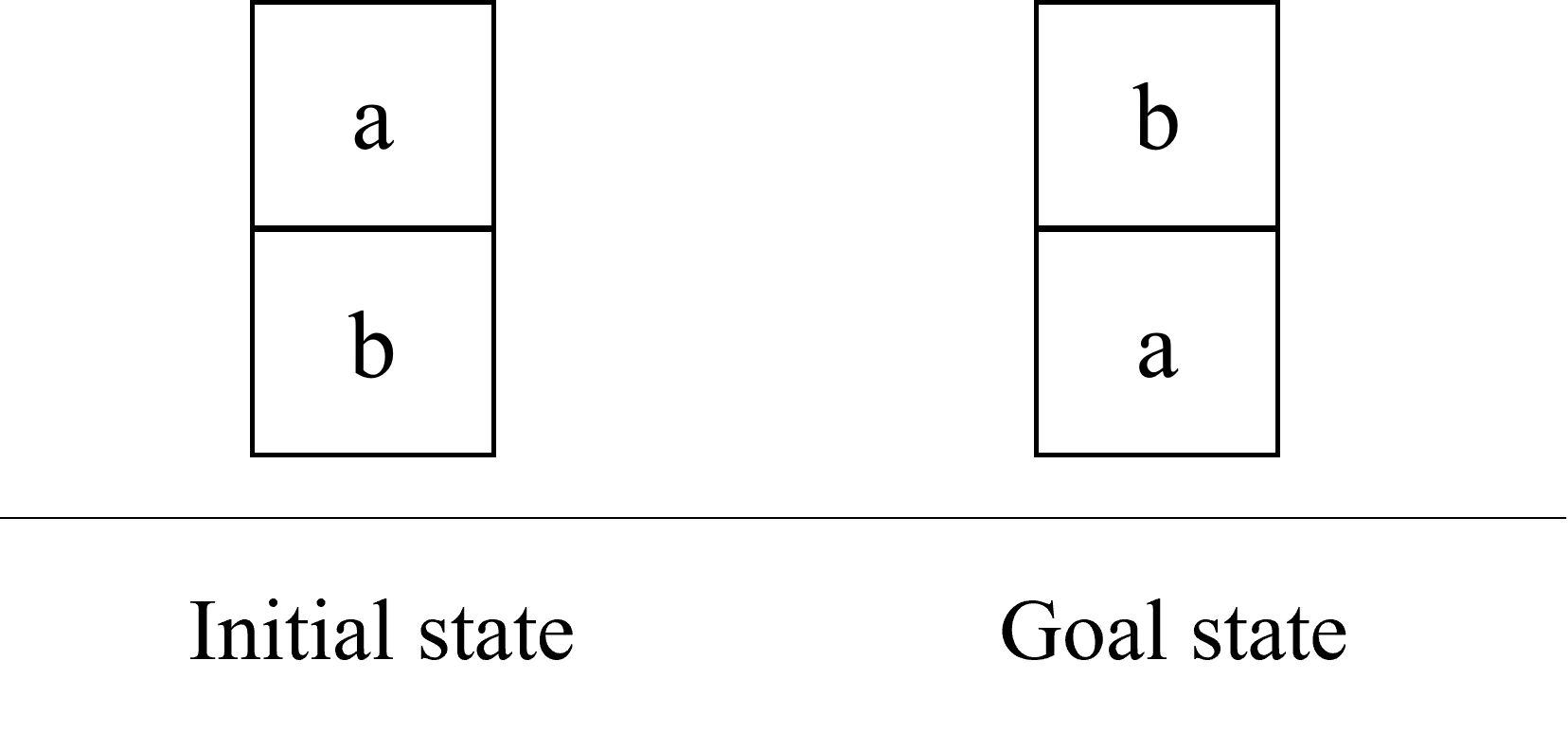}
    \caption{
        The initial and goal states of Blocksworld.
    }\label{fig:planningproblem}
\end{figure}
 
\xasptwo also has the ability to handle explainable planning, meaning it can generate an explanation graph showing why a particular action cannot take place at a certain time. 
To demonstrate this capability, we will use a popular problem known as Blocksworld. The initial state (left) and goal state (right) of the problem are shown in Figure \ref{fig:planningproblem}.
Five fluents are $on(X,Y)$ - block $X$ is on block $Y$, $onTable(X)$ - block $X$ is on the table, $clear(X)$ - block $X$ is clear, $holding(X)$ - the agent holds the block $X$, and $handEmpty$ - the agent does not hold anything.
Four differnt actions are $stack$, $unstack$, $pickup$ and $putdown$.
The domain description of the problem is shown in Table \ref{table:actiondecription} in which the predictions and effects of four actions are presented.

\begin{figure}[t]
\begin{lstlisting}[basicstyle=\ttfamily\scriptsize]
h(X,T+1) :- action(action(A)),occurs(A,T), postcondition(action(A), effect(unconditional),X,value(X,true)).
-h(X,T+1) :- action(action(A)),occurs(A,T), postcondition(action(A), effect(unconditional),X,value(X,false)).
h(X,T+1) :- h(X,T), not -h(X,T+1).
-h(X,T+1) :- -h(X,T), not h(X,T+1).
non_exec(A,T) :- action(action(A)), not h(X,T), precondition(action(A),X,value(X, true)). 
non_exec(A,T) :- action(action(A)), not -h(X,T), precondition(action(A),X,value(X, false)).
:- action(action(A)),occurs(A,T), non_exec(A,T).
\end{lstlisting}

    \caption{ASP program for reasoning about effects of actions \cite{nguyen2020explainable}}\label{fig:actionencoding}
\end{figure}

The rules for reasoning about effects of actions, action generation and goal enforcement \cite{nguyen2020explainable} are utilized as programming input in \xasptwo. Figure \ref{fig:actionencoding} shows the ASP program for reasoning about the effects of actions in which an action occurs only when its preconditions are true and then its effects are true in the next time step. 
Specifically, lines $5$ and $6$ are used to define states in which an action cannot be executed, and constraint is employed to prevent non-executable actions from occurring (line $7$).

For the problem described in Figure \ref{fig:planningproblem}, executing the actions of $unstack(a,b)$, $putdow(a)$, $pickup(b)$ and $stack(b,a)$ at times $0$, $1$, $2$, and $3$ respectively constitutes the optimal plan (assuming time starts at $0$). However, if users are in a rush and want to put down block $a$ on the table at time 0, as represented by the atom $occurs(("putdown", constant("a")), 0)$, they will encounter a false occurrence of the action $putdown(a)$. Figure 6 shows that atom $occurs(("putdown", constant("a")), 0)$ is false because the constraint rule prevents its execution and the prediction of the action holding block a is invalid/false. 

\begin{figure}[!th]
    \centering
    \includegraphics[width=.9\linewidth]{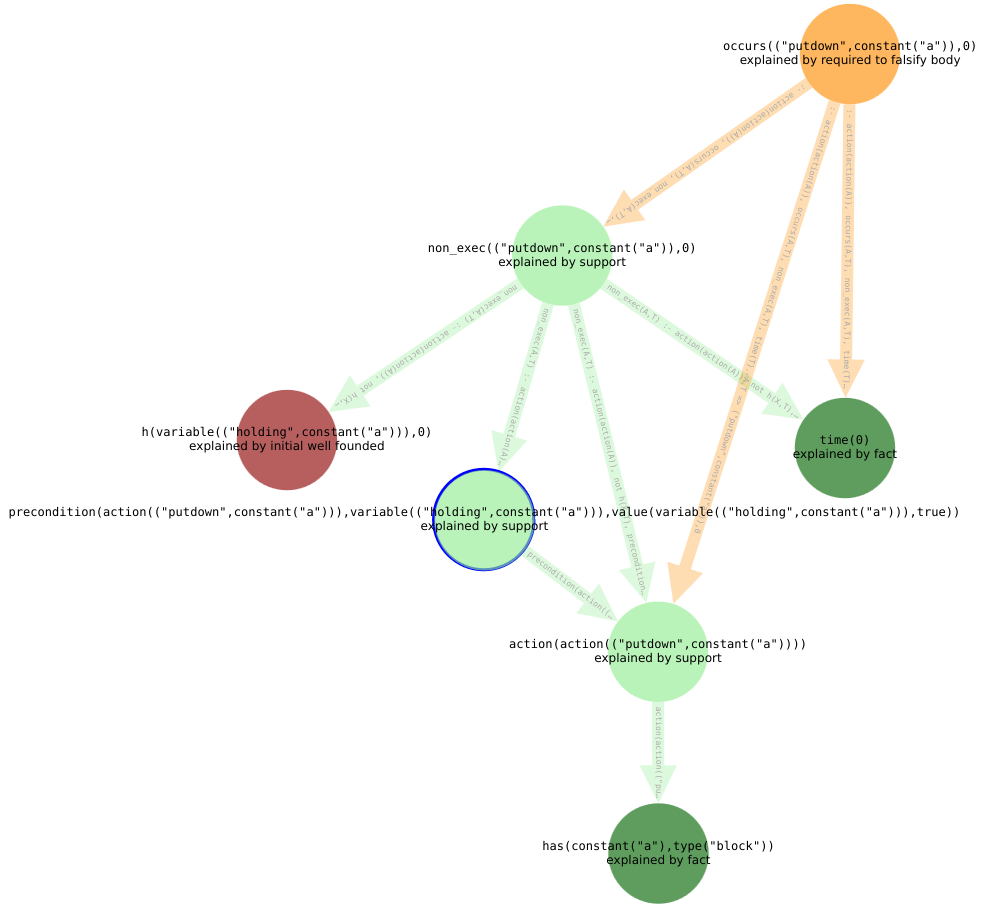}
    \caption{
        The DAG for atom $occurs(("putdown",constant("a")),0)$.
    }\label{fig:putdowna}
\end{figure}

\section{Related Work}\label{sec:rw}
Our work, as stated in the introduction, is situated within the realm of XAI and can be used  as a debugging tool to provide explanations for an unexpected result. For instance, if an element $\alpha$ is false in all answer sets of a program $\Pi$ despite user expectations, our system can help identify which rules are contributing to this anomalous behavior.
Thus, in this section, we will explore both debugging tools for ASP and cutting-edge XAI systems designed for ASP.

In Table~\ref{table}, a summary of the compared features is presented. 
The features include as follows:
whether the explanation is guaranteed to be acyclic;
the capability to handle the input program with aggregates %, choice rules, 
and constraints;
the ability to provide an explanation when the query atom can be false in the answer set; and 
whether the system is available for experimentation. 
As can be seen from Table~\ref{table}, our system is capable of providing explanations for false atoms and does not lead to cyclic argumentation in the explanation. \xasptwo is the only system that tackles a program that includes both aggregates and constraints. 
It is worth noting that the topic of aggregates is addressed in another approach \cite{marynissen2022nested}, even though no system implementing this approach is mentioned or available.

\begin{table}[th]
    \centering
  \caption{Summary of compared features}
  \label{table}
  \begin{tabular}{ *{5}{c} }
    
    \toprule
    \textbf{\shortstack{System (if any) \\ and reference}} & \textbf{\shortstack{Acyclic \\ explanation}} & \textbf{\shortstack{Linguistic \\ extentions}} &
    \textbf{\shortstack{Explanation for \\ false atoms}} &
    \textbf{\shortstack{System \\ availability}} \\
    \midrule
    \scasp{} \cite{scasp2020}   &   Yes  &    Constraints   & Yes 
    & Yes \\
    \system{ASPeRiX}  \cite{beatrix2016justifications}  & Yes   & Constraints    & Yes  & Yes \\
    \spock \cite{brain2007illogical}   &    Yes &   Constraints   & No & Yes \\
    \xclingo{} \cite{Cabalar2020}    &  Yes  &  None     & No &  Yes \\
    \dwasp \cite{dodaro2019debugging}   &   Yes  &   Constraints  & No & Yes \\
     \cite{marynissen2022nested} & No    & Aggregates         & Yes & No \\
    \visualdlv \cite{perri2007integrated}    &   Yes  &  Constraints   & No  & 
    Yes \\
    \cite{pontelli2009justifications}    & No   & None    & Yes  & No \\
    \labas \cite{schulz2016justifying}   & No    & None     & Yes & Yes \\
    \expaspc \cite{trieu2021exp}    & No   & Constraints    & Yes  & Yes \\
    \cite{viegas2013justifications} & Yes    & None    & Yes  & No \\
   
    \xasptwo    & Yes   & Aggregates and Constraints  & Yes & Yes\\
    \bottomrule
  \end{tabular}
\end{table}

\section{Conclusion}\label{sec:conclusion}

We have developed and implemented \xasptwo, an XAI system that targets the ASP language and is powered by ASP engines.
Our approach to explaining why an atom is true/false in an answer set involves deriving easy-to-understand inferences originating from a hopefully small set of atoms assumed false. 
\xasptwo has the ability to support different {\small \tt clingo} constructs such as aggregates and constraints.
It produces an explanation as a DAG with the atom to be explained as the root and takes a few seconds to compute the explanation in our test cases.
Further investigation to include ASP's other linguistic constructs such as conditional literals, beyond those currently supported, present in  the future work.

\section*{Acknowledgments}
Portions of this publication and research effort are made possible through the help and support of NIST via cooperative agreement 70NANB21H167. Tran Cao Son was also partially supported by NSF awards \#1757207 and \#1914635.
Mario Alviano was also partially supported 
by Italian Ministry of Research (MUR) 
    under PNRR project FAIR ``Future AI Research'', CUP H23C22000860006,
    under PNRR project Tech4You ``Technologies for climate change adaptation and quality of life improvement'', CUP H23C22000370006, and
    under PNRR project SERICS ``SEcurity and RIghts in the CyberSpace'', CUP H73C22000880001;
by Italian Ministry of Health (MSAL)
    under POS project RADIOAMICA, CUP H53C22000650006;
by the LAIA lab (part of the SILA labs) and 
by GNCS-INdAM. 

\bibliographystyle{eptcs}
\bibliography{iclp2023}

\begin{thebibliography}{10}
\providecommand{\bibitemdeclare}[2]{}
\providecommand{\surnamestart}{}
\providecommand{\surnameend}{}
\providecommand{\urlprefix}{Available at }
\providecommand{\url}[1]{\texttt{#1}}
\providecommand{\href}[2]{\texttt{#2}}
\providecommand{\urlalt}[2]{\href{#1}{#2}}
\providecommand{\doi}[1]{doi:\urlalt{https://doi.org/#1}{#1}}
\providecommand{\eprint}[1]{arXiv:\urlalt{https://arxiv.org/abs/#1}{#1}}
\providecommand{\bibinfo}[2]{#2}

\bibitemdeclare{article}{scasp2020}
\bibitem{scasp2020}
\bibinfo{author}{Joaquín \surnamestart Arias\surnameend},
  \bibinfo{author}{Manuel \surnamestart Carro\surnameend},
  \bibinfo{author}{Zhuo \surnamestart Chen\surnameend} \&
  \bibinfo{author}{Gopal \surnamestart Gupta\surnameend}
  (\bibinfo{year}{2020}): \emph{\bibinfo{title}{Justifications for
  Goal-Directed Constraint Answer Set Programming}}.
\newblock {\slshape \bibinfo{journal}{Electronic Proceedings in Theoretical
  Computer Science}} \bibinfo{volume}{325}, p. \bibinfo{pages}{59–72},
  \doi{10.4204/EPTCS.325.12}.

\bibitemdeclare{inproceedings}{beatrix2016justifications}
\bibitem{beatrix2016justifications}
\bibinfo{author}{Christopher \surnamestart B{\'e}atrix\surnameend},
  \bibinfo{author}{Claire \surnamestart Lef{\`e}vre\surnameend},
  \bibinfo{author}{Laurent \surnamestart Garcia\surnameend} \&
  \bibinfo{author}{Igor \surnamestart St{\'e}phan\surnameend}
  (\bibinfo{year}{2016}): \emph{\bibinfo{title}{Justifications and blocking
  sets in a rule-based answer set computation}}.
\newblock In: {\slshape \bibinfo{booktitle}{Technical Communications of the
  32nd International Conference on Logic Programming (ICLP 2016)}},
  \bibinfo{organization}{Schloss Dagstuhl-Leibniz-Zentrum fuer Informatik}, pp.
  \bibinfo{pages}{6:1--6:15}, \doi{10.4230/OASIcs.ICLP.2016.6}.

\bibitemdeclare{inproceedings}{brain2007illogical}
\bibitem{brain2007illogical}
\bibinfo{author}{Martin \surnamestart Brain\surnameend},
  \bibinfo{author}{Martin \surnamestart Gebser\surnameend},
  \bibinfo{author}{J{\"o}rg \surnamestart P{\"u}hrer\surnameend},
  \bibinfo{author}{Torsten \surnamestart Schaub\surnameend},
  \bibinfo{author}{Hans \surnamestart Tompits\surnameend} \&
  \bibinfo{author}{Stefan \surnamestart Woltran\surnameend}
  (\bibinfo{year}{2007}): \emph{\bibinfo{title}{That is illogical captain! The
  debugging support tool spock for answer-set programs: system description}}.
\newblock In: {\slshape \bibinfo{booktitle}{Proceedings of the Workshop on
  Software Engineering for Answer Set Programming (SEA’07)}}, pp.
  \bibinfo{pages}{71--85}.

\bibitemdeclare{article}{Cabalar2020}
\bibitem{Cabalar2020}
\bibinfo{author}{Pedro \surnamestart Cabalar\surnameend},
  \bibinfo{author}{Jorge \surnamestart Fandinno\surnameend} \&
  \bibinfo{author}{Brais \surnamestart Muñiz\surnameend}
  (\bibinfo{year}{2020}): \emph{\bibinfo{title}{A System for Explainable Answer
  Set Programming}}.
\newblock {\slshape \bibinfo{journal}{Electronic Proceedings in Theoretical
  Computer Science}} \bibinfo{volume}{325}, p. \bibinfo{pages}{124–136},
  \doi{10.4204/EPTCS.325.19}.

\bibitemdeclare{article}{DBLP:journals/tplp/CalimeriFGIKKLM20}
\bibitem{DBLP:journals/tplp/CalimeriFGIKKLM20}
\bibinfo{author}{Francesco \surnamestart Calimeri\surnameend},
  \bibinfo{author}{Wolfgang \surnamestart Faber\surnameend},
  \bibinfo{author}{Martin \surnamestart Gebser\surnameend},
  \bibinfo{author}{Giovambattista \surnamestart Ianni\surnameend},
  \bibinfo{author}{Roland \surnamestart Kaminski\surnameend},
  \bibinfo{author}{Thomas \surnamestart Krennwallner\surnameend},
  \bibinfo{author}{Nicola \surnamestart Leone\surnameend},
  \bibinfo{author}{Marco \surnamestart Maratea\surnameend},
  \bibinfo{author}{Francesco \surnamestart Ricca\surnameend} \&
  \bibinfo{author}{Torsten \surnamestart Schaub\surnameend}
  (\bibinfo{year}{2020}): \emph{\bibinfo{title}{ASP-Core-2 Input Language
  Format}}.
\newblock {\slshape \bibinfo{journal}{Theory Pract. Log. Program.}}
  \bibinfo{volume}{20}(\bibinfo{number}{2}), pp. \bibinfo{pages}{294--309},
  \doi{10.1017/s1471068419000450}.

\bibitemdeclare{article}{dodaro2019debugging}
\bibitem{dodaro2019debugging}
\bibinfo{author}{Carmine \surnamestart Dodaro\surnameend},
  \bibinfo{author}{Philip \surnamestart Gasteiger\surnameend},
  \bibinfo{author}{Kristian \surnamestart Reale\surnameend},
  \bibinfo{author}{Francesco \surnamestart Ricca\surnameend} \&
  \bibinfo{author}{Konstantin \surnamestart Schekotihin\surnameend}
  (\bibinfo{year}{2019}): \emph{\bibinfo{title}{Debugging non-ground ASP
  programs: Technique and graphical tools}}.
\newblock {\slshape \bibinfo{journal}{Theory and Practice of Logic
  Programming}} \bibinfo{volume}{19}(\bibinfo{number}{2}), pp.
  \bibinfo{pages}{290--316}, \doi{10.1017/S1471068418000492}.

\bibitemdeclare{inproceedings}{GelfondL90}
\bibitem{GelfondL90}
\bibinfo{author}{M.~\surnamestart Gelfond\surnameend} \&
  \bibinfo{author}{V.~\surnamestart Lifschitz\surnameend}
  (\bibinfo{year}{1990}): \emph{\bibinfo{title}{Logic programs with classical
  negation}}.
\newblock In \bibinfo{editor}{D.~\surnamestart Warren\surnameend} \&
  \bibinfo{editor}{Peter \surnamestart Szeredi\surnameend}, editors: {\slshape
  \bibinfo{booktitle}{Logic Programming: Proc.~of the Seventh International
  Conference}}, pp. \bibinfo{pages}{579--597}.

\bibitemdeclare{article}{kaminskiromeroschaubwanko2023}
\bibitem{kaminskiromeroschaubwanko2023}
\bibinfo{author}{Roland \surnamestart Kaminski\surnameend},
  \bibinfo{author}{Javier \surnamestart Romero\surnameend},
  \bibinfo{author}{Torsten \surnamestart Schaub\surnameend} \&
  \bibinfo{author}{Philipp \surnamestart Wanko\surnameend}
  (\bibinfo{year}{2023}): \emph{\bibinfo{title}{How to Build Your Own ASP-based
  System?!}}
\newblock {\slshape \bibinfo{journal}{Theory and Practice of Logic
  Programming}} \bibinfo{volume}{23}(\bibinfo{number}{1}), p.
  \bibinfo{pages}{299–361}, \doi{10.1017/S1471068421000508}.

\bibitemdeclare{inproceedings}{li2021discasp}
\bibitem{li2021discasp}
\bibinfo{author}{Fang \surnamestart Li\surnameend}, \bibinfo{author}{Huaduo
  \surnamestart Wang\surnameend}, \bibinfo{author}{Kinjal \surnamestart
  Basu\surnameend}, \bibinfo{author}{Elmer \surnamestart Salazar\surnameend} \&
  \bibinfo{author}{Gopal \surnamestart Gupta\surnameend}
  (\bibinfo{year}{2021}): \emph{\bibinfo{title}{DiscASP: A Graph-based ASP
  System for Finding Relevant Consistent Concepts with Applications to
  Conversational Socialbots}}.
\newblock In: {\slshape \bibinfo{booktitle}{Proceedings 37th International
  Conference on Logic Programming (Technical Communications), {ICLP} Technical
  Communications 2021, Porto (virtual event), 20-27th September 2021}},
  {\slshape \bibinfo{series}{{EPTCS}}} \bibinfo{volume}{345}, pp.
  \bibinfo{pages}{205--218}, \doi{10.4204/EPTCS.345.35}.

\bibitemdeclare{inproceedings}{MarekT99}
\bibitem{MarekT99}
\bibinfo{author}{V.~\surnamestart Marek\surnameend} \&
  \bibinfo{author}{M.~\surnamestart Truszczy\'nski\surnameend}
  (\bibinfo{year}{1999}): \emph{\bibinfo{title}{Stable models and an
  alternative logic programming paradigm}}.
\newblock In: {\slshape \bibinfo{booktitle}{The Logic Programming Paradigm: a
  25-year Perspective}}, pp. \bibinfo{pages}{375--398},
  \doi{10.1007/978-3-642-60085-2_17}.

\bibitemdeclare{article}{marynissen2022nested}
\bibitem{marynissen2022nested}
\bibinfo{author}{Simon \surnamestart Marynissen\surnameend},
  \bibinfo{author}{Jesse \surnamestart Heyninck\surnameend},
  \bibinfo{author}{Bart \surnamestart Bogaerts\surnameend} \&
  \bibinfo{author}{Marc \surnamestart Denecker\surnameend}
  (\bibinfo{year}{2022}): \emph{\bibinfo{title}{On nested justification systems
  (full version)}}.
\newblock {\slshape \bibinfo{journal}{CoRR}} \bibinfo{volume}{abs/2205.04541},
  \doi{10.48550/arXiv.2205.04541}.

\bibitemdeclare{inproceedings}{nguyen2020explainable}
\bibitem{nguyen2020explainable}
\bibinfo{author}{Van \surnamestart Nguyen\surnameend},
  \bibinfo{author}{Stylianos~Loukas \surnamestart Vasileiou\surnameend},
  \bibinfo{author}{Tran~Cao \surnamestart Son\surnameend} \&
  \bibinfo{author}{William \surnamestart Yeoh\surnameend}
  (\bibinfo{year}{2020}): \emph{\bibinfo{title}{Explainable planning using
  answer set programming}}.
\newblock In: {\slshape \bibinfo{booktitle}{Proceedings of the International
  Conference on Principles of Knowledge Representation and Reasoning}}, pp.
  \bibinfo{pages}{662--666}, \doi{10.24963/kr.2020/66}.

\bibitemdeclare{article}{Niemela99}
\bibitem{Niemela99}
\bibinfo{author}{I.~\surnamestart Niemel{\"{a}}\surnameend}
  (\bibinfo{year}{1999}): \emph{\bibinfo{title}{Logic programming with stable
  model semantics as a constraint programming paradigm}}.
\newblock {\slshape \bibinfo{journal}{Annals of Mathematics and Artificial
  Intelligence}} \bibinfo{volume}{25}(\bibinfo{number}{3,4}), pp.
  \bibinfo{pages}{241--273}, \doi{10.1023/A:1018930122475}.

\bibitemdeclare{article}{DBLP:journals/tplp/PelovDB07}
\bibitem{DBLP:journals/tplp/PelovDB07}
\bibinfo{author}{Nikolay \surnamestart Pelov\surnameend}, \bibinfo{author}{Marc
  \surnamestart Denecker\surnameend} \& \bibinfo{author}{Maurice \surnamestart
  Bruynooghe\surnameend} (\bibinfo{year}{2007}):
  \emph{\bibinfo{title}{Well-founded and stable semantics of logic programs
  with aggregates}}.
\newblock {\slshape \bibinfo{journal}{Theory Pract. Log. Program.}}
  \bibinfo{volume}{7}(\bibinfo{number}{3}), pp. \bibinfo{pages}{301--353},
  \doi{10.1017/S1471068406002973}.

\bibitemdeclare{inproceedings}{perri2007integrated}
\bibitem{perri2007integrated}
\bibinfo{author}{Simona \surnamestart Perri\surnameend},
  \bibinfo{author}{Francesco \surnamestart Ricca\surnameend},
  \bibinfo{author}{Giorgio \surnamestart Terracina\surnameend},
  \bibinfo{author}{D~\surnamestart Cianni\surnameend} \&
  \bibinfo{author}{P~\surnamestart Veltri\surnameend} (\bibinfo{year}{2007}):
  \emph{\bibinfo{title}{An integrated graphic tool for developing and testing
  DLV programs}}.
\newblock In: {\slshape \bibinfo{booktitle}{Proceedings of the Workshop on
  Software Engineering for Answer Set Programming (SEA’07)}}, pp.
  \bibinfo{pages}{86--100}.

\bibitemdeclare{article}{pontelli2009justifications}
\bibitem{pontelli2009justifications}
\bibinfo{author}{Enrico \surnamestart Pontelli\surnameend},
  \bibinfo{author}{Tran~Cao \surnamestart Son\surnameend} \&
  \bibinfo{author}{Omar \surnamestart Elkhatib\surnameend}
  (\bibinfo{year}{2009}): \emph{\bibinfo{title}{Justifications for logic
  programs under answer set semantics}}.
\newblock {\slshape \bibinfo{journal}{Theory and Practice of Logic
  Programming}} \bibinfo{volume}{9}(\bibinfo{number}{1}), pp.
  \bibinfo{pages}{1--56}, \doi{10.1017/S1471068408003633}.

\bibitemdeclare{article}{schulz2016justifying}
\bibitem{schulz2016justifying}
\bibinfo{author}{Claudia \surnamestart Schulz\surnameend} \&
  \bibinfo{author}{Francesca \surnamestart Toni\surnameend}
  (\bibinfo{year}{2016}): \emph{\bibinfo{title}{Justifying answer sets using
  argumentation}}.
\newblock {\slshape \bibinfo{journal}{Theory and Practice of Logic
  Programming}} \bibinfo{volume}{16}(\bibinfo{number}{1}), pp.
  \bibinfo{pages}{59--110}, \doi{10.1017/S1471068414000702}.

\bibitemdeclare{article}{trieu2021exp}
\bibitem{trieu2021exp}
\bibinfo{author}{Ly~Ly \surnamestart Trieu\surnameend},
  \bibinfo{author}{Tran~Cao \surnamestart Son\surnameend} \&
  \bibinfo{author}{Marcello \surnamestart Balduccini\surnameend}
  (\bibinfo{year}{2021}): \emph{\bibinfo{title}{exp(aspc): explaining asp
  programs with choice atoms and constraint rules}}.
\newblock {\slshape \bibinfo{journal}{Electronic Proceedings in Theoretical
  Computer Science}}, \doi{10.4204/EPTCS.345.28}.

\bibitemdeclare{inproceedings}{trieu2022explanation}
\bibitem{trieu2022explanation}
\bibinfo{author}{Ly~Ly \surnamestart Trieu\surnameend},
  \bibinfo{author}{Tran~Cao \surnamestart Son\surnameend} \&
  \bibinfo{author}{Marcello \surnamestart Balduccini\surnameend}
  (\bibinfo{year}{2022}): \emph{\bibinfo{title}{xASP: An Explanation Generation
  System for Answer Set Programming}}.
\newblock In: {\slshape \bibinfo{booktitle}{International Conference on Logic
  Programming and Nonmonotonic Reasoning}}, \bibinfo{organization}{Springer},
  pp. \bibinfo{pages}{363--369}, \doi{10.1007/978-3-031-15707-3\_28}.

\bibitemdeclare{inproceedings}{viegas2013justifications}
\bibitem{viegas2013justifications}
\bibinfo{author}{Carlos \surnamestart Viegas~Dam{\'a}sio\surnameend},
  \bibinfo{author}{Anastasia \surnamestart Analyti\surnameend} \&
  \bibinfo{author}{Grigoris \surnamestart Antoniou\surnameend}
  (\bibinfo{year}{2013}): \emph{\bibinfo{title}{Justifications for logic
  programming}}.
\newblock In: {\slshape \bibinfo{booktitle}{Logic Programming and Nonmonotonic
  Reasoning: 12th International Conference, LPNMR 2013, Corunna, Spain,
  September 15-19, 2013. Proc.~ 12}}, \bibinfo{organization}{Springer}, pp.
  \bibinfo{pages}{530--542}, \doi{10.1007/978-3-642-40564-8\_53}.

\end{thebibliography}

% \clearpagez

\end{document}